\documentclass[twoside]{article}
\usepackage[accepted]{aistats2018}
\usepackage{amsmath, mathtools, xpatch}
\usepackage{bm}
\usepackage{multirow}
\usepackage{algorithm}
\usepackage[noend]{algpseudocode}
\usepackage{amsthm}
\usepackage{natbib}
\usepackage{subcaption}
\usepackage{enumitem,amsmath,amssymb,graphicx,cleveref,mathtools,nicefrac,bbm,todonotes,dsfont,caption, xcolor,multirow, soul}
\usepackage{mwe}
\usepackage{algpseudocode}
\usepackage{algorithm}

\def\indic#1{\mathbb{I}\left({#1}\right)} % Indicator function

\DeclareMathOperator*{\argmax}{arg\,max}

\makeatletter
\xpatchcmd{\proof}{\topsep6\p@\@plus6\p@\relax}{}{}{}
\makeatother

\def\R{\mathbb{R}}

\def\E{\mathbb{E}}
\makeatletter
\def\BState{\State\hskip-\ALG@thistlm}
\makeatother

\def\Esubarg#1#2{\E_{#1}\left[{#2}\right]}

\newcommand{\Prob}{\Pr}
\def\Parg#1{\Prob\left[{#1}\right]}
\def\Psubarg#1#2{\Prob_{#1}\left[{#2}\right]}

\newcommand{\bfU}{\mathbf{U}}

\newtheorem{theorem}{Theorem}
\newtheorem{lemma}{Lemma}

\newtheorem{proposition}{Proposition}
\newtheorem{definition}{Definition}

\newcommand*\mean[1]{\overline{#1}}

\usepackage[utf8]{inputenc} % allow utf-8 input
\usepackage[T1]{fontenc}    % use 8-bit T1 fonts
\usepackage{hyperref}       % hyperlinks
\usepackage{url}            % simple URL typesetting
\usepackage{booktabs}       % professional-quality tables
\usepackage{amsfonts}       % blackboard math symbols
\usepackage{nicefrac}       % compact symbols for 1/2, etc.
\usepackage{microtype}      % microtypography

\begin{document}



\twocolumn[

\aistatstitle{Why Adaptively Collected Data Have Negative Bias and How to Correct for It}

\aistatsauthor{ Xinkun Nie \And Xiaoying Tian \And  Jonathan Taylor \And James Zou }

\aistatsaddress{ Stanford University \And  Stanford University \And Stanford University \And Stanford University } ]

\begin{abstract}
From scientific experiments to online A/B testing, the previously observed data often affects how future experiments are performed, which in turn affects which data will be collected. Such adaptivity introduces complex correlations between the data and the collection procedure. In this paper, we prove that when the data collection procedure satisfies natural conditions, then sample means of the data have systematic \emph{negative} biases. As an example, consider an adaptive clinical trial where additional data points are more likely to be tested for treatments that show initial promise. Our surprising result implies that the average observed treatment effects would underestimate the true effects of each treatment. We quantitatively analyze the magnitude and behavior of this negative bias in a variety of settings. We also propose a novel debiasing algorithm based on selective inference techniques. In experiments, our method can effectively reduce bias and estimation error. 
\end{abstract}

\section{INTRODUCTION}

Much of modern data science is driven by data that is collected adaptively. A scientist often starts off testing multiple experimental conditions, and based on the initial results may decide to collect more data points from some conditions and less data from other settings. A sequential clinical trial initially groups the participants into different treatment regimes, and depending on the continuous feedback, may reallocate participants into the more promising treatments. In e-commerce, companies often use online A/B tests to collect user data from multiple variants of a project, and could adaptively collect more data from a subset of the variants (multi-arm bandit algorithms are often used here to decide which variant to collect data from as a function of the data log history). 

The key characteristic of adaptively collected data is that the analyst sequentially collects data from multiple alternatives (e.g. different treatments, products, etc.). The choice of which alternative to gather data from at a particular time depends on the previously observed data from all the options. The collected data could be used in many different ways. In some settings, the analyst simply wants to use it to identify the single best alternative, and may not care about the data beyond this goal (this setting motivates many bandit problems). In many other settings, the data itself could be used to estimate various statistical parameters. In the sequential clinical trial example, many scientists would like to use the data to estimate the effects of each of the treatments. Even if the company sponsoring the trials may care most about identifying the best treatment, other scientist using the data may care about the effect size estimates of other treatments in the data for their own applications. 

\paragraph{Our contributions.} We study the problem of estimation using adaptively collected data. We prove that when the adaptive data collection procedure satisfies two natural conditions (precisely defined in Sec.~\ref{sec:neg_bias}), then the sample mean of the collected data is negatively biased as an estimator for the true mean. This means that the effect size empirically observed is systematically less than the true effect size for every alternative. We provide intuition for this counter-intuitive result, and compare and analyze the magnitude of this negative bias across different conditions and collection procedures. We then propose a novel randomized algorithm called the conditional Maximum Likelihood Estimator (cMLE) based on selective inference to reduce this ubiquitous bias, and compare it with a simple approach using an independent set of held-out data. We validate the performance of our bias-reduction algorithm in extensive experiments. All the proofs and additional experiments are in the Appendix. 

\paragraph{Related works.}

Multi-arm bandits and its variations are extensively studied in machine learning. The goal of our work is different from that of the standard bandit setting. In bandits, the data sampled from an arm (i.e. one of the alternatives) is considered a reward and the objective is to \emph{design} adaptive algorithms to pick arms so as to maximize total reward (or minimize regret). Our goal is not to design such algorithms and we are agnostic to the reward. We take the perspective of an analyst who is given such an adaptively collected dataset and wants to estimate statistical parameters. 

\citet{xu2013estimation} empirically observed estimation bias due to selection in specific multi-arm bandit algorithms. They were primarily interested in estimating the values of the top two arms, and used data splitting with a held-out set in their experiments to reduce bias. We are the first one to rigorously prove that such underestimation is a general phenomenon. Our cMLE approach builds upon recent advances in selective inference \citep{taylor2015statistical,randomized_response}, which derives valid confidence intervals accounting for selection effects of the algorithm. Selective inference has been applied to regression problems (e.g. LASSO, Stepwise regression), and has not been considered for the adaptive data collection setting before. We build upon results from recent developments in this area \citep{randomized_response,tian2016magic,selective_sampling}. 

The problem of selection bias has been extensively studied, especially in the context of Winner's Curse in genetic association studies \citep{ionita2009genetic}. There the bias arises from selective reporting rather than adaptivity in data collection. There is a related line of recent work \citep{reusable_holdout,russo2016controlling} in adaptive data analysis that is complementary to ours. In their work the data is fixed (and is typically i.i.d.) and the adaptivity is in the analyst. In contrast, in our work the data collection itself is adaptive.

\section{ADAPTIVE DATA COLLECTION HAS NEGATIVE BIAS}\label{sec:neg_bias}

\paragraph{Model of adaptive data collection.} We have $K$ unknown distributions  that we would like to collect data from. There are $T$ rounds of data collection and at round $t\in [T]$ the distribution $s_t \in [K]$ is selected, and we draw $X^{(s_t)}$, an independent sample, from $s_t$. 
 The data collection procedure can be modeled by a selection function $s_t = f(\Lambda_t)$, where $\Lambda_t$ is the history of the observed samples up to time $t$. More precisely, let $X^{(k)}_i$ denote the $i$-th sample from distribution $k$ and $N_{t}^{(k)}$ denote the number of times that distribution $k$ is sampled by round $t$, which could be a random variable, then $\Lambda_t = \{ \{X_1^{(1)}, ..., X_{N^{(1)}_t}^{(1)}\}, ...,   \{X_1^{(K)}, ..., X_{N^{(K)}_t}^{(K)}\}\}$. The history of distribution $k$ up to round $t$ is denoted by $\Lambda_t^{(k)} = \{X_1^{(k)}, ..., X_{N^{(K)}_t}^{(k)}\}$. We use $\Lambda_t^{(-k)}$ to denote the history up to round $t$ of all the distributions except for the $k$-th one; $\Lambda_t^{(-k)} = \{  \{X_1^{(i)}, ..., X_{N^{(K)}_t}^{(i)}\}\}_{i \in [K]\setminus k}$. 
We allow $f$ to be a randomized function, and will sometimes write $f(\Lambda_t, \omega)$, where $\omega \in \Omega$ is a random seed, to highlight this randomness. Let $\mean{X_t^{(k)}} \equiv \frac{\sum_{i=1}^{N_t^{(k)}} X_i^{(k)}}{N_t^{(k)}}$ denote the sample average of distribution $k$ at round $t$. 
\paragraph{Example.}  The simplest example of adaptive data collection is the Greedy algorithm. In Greedy, at round $t$, the selection function chooses to sample the distribution with the highest empirical mean.  Then $f(\Lambda_t) = \argmax_{k\in [K]} \mean{X_t^{(k)}}$. Often in practice, a randomized version of Greedy, called $\epsilon$-Greedy, is also used. In $\epsilon$-Greedy with probability $\epsilon$ we uniformly randomly select a distribution and with probability $1-\epsilon$, we perform Greedy. This corresponds to the selection  
$$
f(\Lambda_t, \omega) = \begin{cases}
\argmax_{k\in [K]} \mean{X_t^{(k)}},  \text{ if } \omega > \epsilon \\
k, k \in [K], \text{ if } \frac{\epsilon}{K} \cdot (k-1) < \omega < \frac{\epsilon}{K} \cdot k
\end{cases}
$$
where $\omega \sim \mbox{Unif}[0,1].$ All common multi-arm bandit algorithms can be modeled as a selection function $f$.

 Many adaptive data collection procedures correspond to a selection function $f$ that satisfies two natural properties: \emph{Exploit} and \emph{Independence of Irrelevant Option (IIO)}. 
\emph{Exploit} means that all else being equal, if distribution $k$ is selected in a scenario where it has lower sample average, then $k$ would also be selected in a scenario where it has higher sample average. \emph{IIO} means that if distribution $k$ is not selected then the precise values observed from $k$ does not affect which of the other distributions is selected. We precisely define these two properties next. 

\begin{definition} [Exploit]
Given any $t \in [T]$, $k \in [K]$, realization $\Lambda_t^{(-k)} $ and random seed $\omega$.  Suppose $\Lambda_t^{(k)}$ and $\Lambda_t^{'(k)}$ are two sample histories of distribution $k$ of length $n$ with sample means $\mean{X_t^{(k)}} \leq \mean{X_t^{'(k)}}$. Then $f(\Lambda_t^{(k)}\cup \Lambda_t^{(-k)}, \omega ) = k$ implies $f(\Lambda_t^{'(k)}\cup \Lambda_t^{(-k)}, \omega ) = k$. In words, \emph{Exploit} states that given the same context specified by $\Lambda_t^{(-k)}$ and  $\omega$, if $k$ is selected when it has smaller sample mean then it should also be selected when it has a larger mean. 
\end{definition}
\vspace{0.05cm}

\emph{Exploit} captures the intuition that when we are looking for options that work well, we are more likely to try out the options that show more promise early on. 
Note that in \emph{Exploit}, we only compare two sample histories $\Lambda_t^{(k)}$ and $\Lambda_t^{'(k)}$ with the same number of observed samples. This allows $f$ to also account for the number of samples observed so far (e.g. selecting a distribution $k$ with low sample average if it does not have many samples). Therefore confidence interval based bandit algorithms can also be shown to satisfy \emph{Exploit}.   
\vspace{0.05cm}

\begin{definition} [Independent of Irrelevant Options (IIO)] \label{def:iio}
Given any $t \in [T]$ and $k \in [K]$. Let $\Lambda_t = \Lambda_t^{(k)}\cup \Lambda_t^{(-k)}$ and $\Lambda'_t = \Lambda_t^{'(k)}\cup \Lambda_t^{(-k)}$, i.e. $\Lambda_t$ and $\Lambda'_t$ have the same histories for distributions $i \neq k$ and could have arbitrary histories for distribution $k$. Then $\forall$ $i \neq k$, 
\[
\Pr\left[f\left(\Lambda_t\right) = i | f\left(\Lambda_t\right) \neq k \right] =  \Pr\left[f\left(\Lambda'_t\right) = i | f\left(\Lambda'_t\right) \neq k \right].
\]
In words, so long as $k$ is not chosen, which other distribution is selected depends only on the history $\Lambda_t^{(-k)}$ of those distributions. 
\end{definition}

\paragraph{Estimation bias.} In this paper, we are interested in the fundamental problem of estimating the true mean, $\mu_k = \E[X^{(k)}]$, of each of the distributions given a sample history dataset, $\Lambda_T$, which is collected through an adaptive procedure. This models the adaptive clinical trials example, where the scientist is interested in estimating $\{\mu_k\}_{k\in[K]}$, the true effects of the treatments. Of course, if the scientist can collect her own data, she could just collect a non-adaptive set of samples and obtain unbiased estimates of $\{\mu_k\}_{k\in[K]}$. However, in many settings like the clinical trials, the scientist does not collect the data; rather it is adaptively collected by a pharmaceutical company with a different objective of finding an optimal treatment or demonstrating efficacy. The simplest and most common approach is to use the sample average $\mean{X^{(k)}_T}$ to estimate the true mean $\mu_k$. Our main result shows that in expectation, the sample average underestimates the true mean if $f$ satisfies \emph{Exploit} and \emph{IIO}: $\E \left[\mean{X_T^{(k)}} \right] \leq \mu_k, \forall k\in [K]$.

\begin{theorem} \label{thm:negative_bias}
Suppose $X^{(k)},k\in[K]$ is a sample drawn from a distribution with finite mean $\mu_k = \E[X^{(k)}]$, and the selection function $f$ satisfies \emph{Exploit} and \emph{IIO}. Then $\forall k$ and $\forall T$, $\E \left[\mean{X_T^{(k)}} \right] \leq \mu_k$. Moreover, the equality holds only if the number of times distribution $k$ is selected, $N^{(k)}_T$, does not depend on the observed history $\Lambda_T^{(k)}$ of $k$. 
\end{theorem}

\begin{proof}[Intuition behind the proof.] Here we present the high-level insights for the proof. The detailed proof is in Appendix~\ref{app:proofs}.
For simplicity, we condition on a fixed realization of distributions $2, \dots, K$.
If the bias of distribution 1 is negative for every realization of distributions $2, \dots, K$, then taking the expectation shows that the total bias is negative. 

Consider a particular sample path history $\Lambda_t^{(1)}$ at some round $t< T$, with corresponding empirical average $\mean{X^{(1)}_t}$. There are two types of scenarios. First, $\Lambda_t^{(1)}$ could be \emph{lucky} and $\mean{X^{(1)}_t} > \mu_1$. Then the \emph{Exploit} property states that with this lucky sample path history, distribution 1 is likely to be sampled more often in the future in rounds $[t+1, \dots, T]$. Since these future samples have expected values $\mu_1$, the expected average of final value of $\mean{X^{(1)}_T}$ is likely to decrease closer to $\mu_1$. This is similar to the reversion to mean phenomenon. In scenario two, $\Lambda_t^{(1)}$ is \emph{unlucky} and $\mean{X^{(1)}_t} < \mu_1$. The exploitative nature of $f$ makes it less likely to select distribution 1 and this sample path history is likely to be stuck with the negative bias. Therefore we see that the exploitative property of the adaptive collection procedure creates a fundamental asymmetry in the sample path histories such that the positive bias (lucky) paths revert back to mean but the negative bias (unlucky) paths are stuck at negative. The overall bias becomes negative. The \emph{IIO} property allows us to safely condition on the realizations of distributions $2, \dots, K$ and isolate the effects of distribution 1. 
\end{proof}
Many standard multi-arm bandit algorithms can be modeled by a selection function $f$ that satisfies \emph{Exploit} and \emph{IIO}.  While Greedy only has sample mean as its input, upper confidence bound (UCB) type algorithms also account for the number of observations and give preference for the less explored distributions. lil' UCB is the state-of-the-art UCB algorithm \citep{jamieson2014lil} and its details are presented in Appendix~\ref{app:lil' UCB}. 
\begin{proposition}\label{prop:algs_satifies}
 lil' UCB, Greedy, $\epsilon$-Greedy are all equivalent to selection functions $f(\Lambda_t)$ that satisfy \emph{Exploit} and \emph{IIO}. 
\end{proposition}
In Appendix~\ref{app:thompson}, we extend Proposition~\ref{prop:algs_satifies} to Thompson Sampling \citep{thompson1933likelihood,agrawal2012analysis}. When $K = 2$, we do not need the \emph{IIO} condition in order for the bias to be non-positive. 
\begin{proposition} \label{prop:K=2}
Suppose $X^{(1)}, X^{(2)}$ are samples drawn from distributions with finite means $\mu_1, \mu_2$ and the selection function $f$ satisfies \emph{Exploit}. Then for $k \in \{1, 2\}$ and all $T$, $\E \left[\mean{X_T^{(k)}} \right] \leq \mu_k$. Moreover the equality holds only if the number of times distribution $k$ is selected, $N^{(k)}_T$, does not depend on observed values $\Lambda_T^{(k)}$ of $k$. 
\end{proposition}
\section{QUANTITATIVE CHARACTERIZATION OF BIAS}
\label{sec:bias-analysis}
\paragraph{Analytic example with explicit bias.} Consider the setting where $K=2$, $X^{(1)} \sim Bernoulli(\mu_1)$ and $X^{(2)} \sim Bernoulli(\mu_2)$. A greedy data collection procedure is to draw one sample from each distribution in the first two rounds, and at $T=3$ sample from the distribution with the larger empirical sample mean. In the event of a tie, i.e. both samples are 0 or 1, $X^{(1)}$ is selected by default. We can derive analytic expressions for the bias of the empirical mean of each distribution at $T = 3$.    
\begin{eqnarray*}\label{eq:example_bias}
 \mbox{bias}_1 \equiv  \E\left[ \mean{X^{(1)}_3} \right] - \mu_1 &=& -\frac{1}{2}\mu_1(1-\mu_1) \mu_2,  \\
\mbox{bias}_2 \equiv \E\left[ \mean{X^{(2)}_3} \right] - \mu_2 &=& -\frac{1}{2}\mu_2(1-\mu_2)(1- \mu_1).
\end{eqnarray*}
When $0 < \mu_1, \mu_2 < 1$, both biases are strictly negative. 
This simple example already demonstrates the interesting phenomenon that \emph{the distribution with the highest mean does not always have the least bias}. Using the above analytical forms, the ratio of the biases is $\frac{\mbox{bias}_1}{\mbox{bias}_2} = \frac{\mu_1}{1-\mu_2}$. Therefore $\mbox{bias}_2$ is worse than $\mbox{bias}_1$ when $\mu_1$, $\mu_2$ are both close to 1, and $\mbox{bias}_1$ is worse than $\mbox{bias}_2$ when $\mu_1$, $\mu_2$ are both close to 0. Figure~\ref{fig:bias-sim} illustrates the quantitative bias of the first distribution $X^{(1)}$ at times $T = 3$ and $T=10$. The setup and bias is symmetric for the second distribution $X^{(2)}$.

%This point is further illustrated empirically in Figure~\ref{fig:bias}(d) in the Gaussian case. 
%Following this greedy algorithm, we plot ${\mbox{bias}_1}$ in Figure~\ref{fig:bias-sim} as a function of $\mu_1$ and $\mu_2$ at $T=3$ and $T=10$. Note that the scales of the two plots are different. Note that ${\mbox{bias}_1}$ is a complicated function of $\mu_1$ and $\mu_2$ at $T=10$, but remains non-positive. 
\begin{figure}[t]
  \centering
  \begin{subfigure}{.49\columnwidth}
    \centering
    \includegraphics[width=\linewidth]{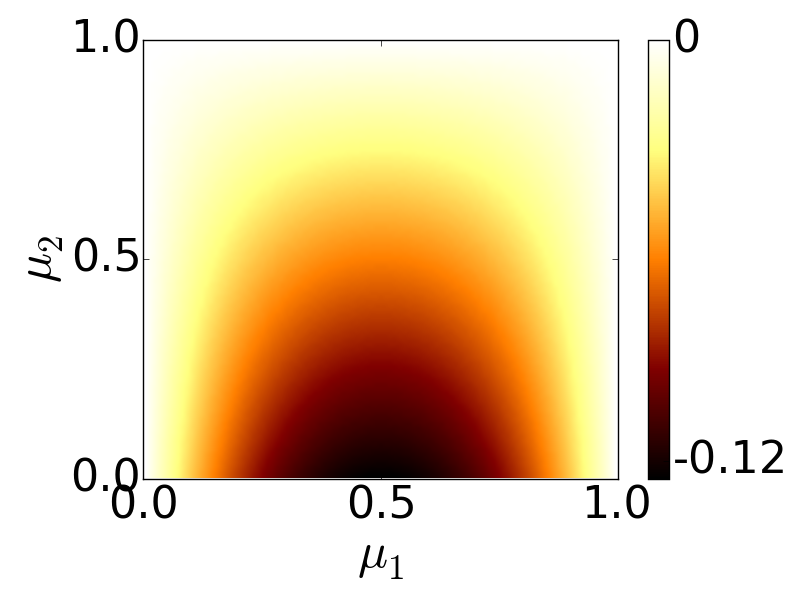}
    \caption{T=3}
  \end{subfigure}%
  \hfill
  \begin{subfigure}{.49\columnwidth}
    \centering
    \includegraphics[width=\linewidth]{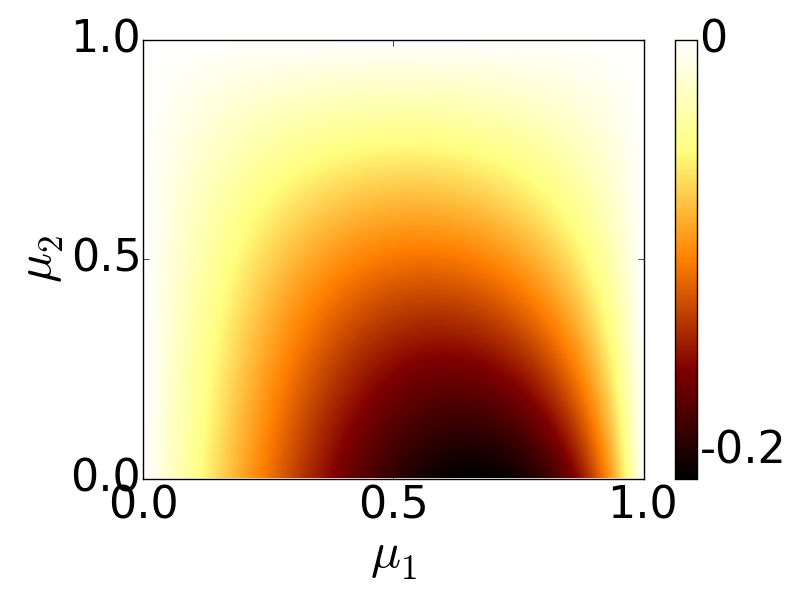}
    \caption{T=10}
  \end{subfigure}
  \caption{The bias of $X^{(1)}$ as a function of $\mu_1$ and $\mu_2$ running the Greedy algorithm, at $T=3$ and $T=10$ respectively. Note the difference in scale.}
  \label{fig:bias-sim}
\end{figure}

The insight from our proof of Theorem~\ref{thm:negative_bias} is that the bias of distribution $k$ at time $t$ should be large if how likely we are to choose $k$ in the future (after $t$) is sensitive to the value $\mean{X^{(k)}_t}$.  This sensitivity increases if there is \emph{consequential competition} for distribution $k$ at time $t$, i.e. if there are other distribution(s), $i$, whose empirical average $\mean{X^{(i)}_t}$ is in some \emph{consequential} middle range from the empirical average of distribution $k$. When they are too far apart, the particular sample values drawn from $k$ are not consequential to the chance of it getting sampled again. If they are too close, having one bad sample value also does not affect the chance of $k$ being drawn as much. We demonstrate the above remarks empirically in the next section.
\begin{figure*}[t]
  \centering
  \subcaptionbox{lil' UCB, $\mu$ scale = 1}[.31\linewidth][c]{%
    \includegraphics[width=.31\linewidth]{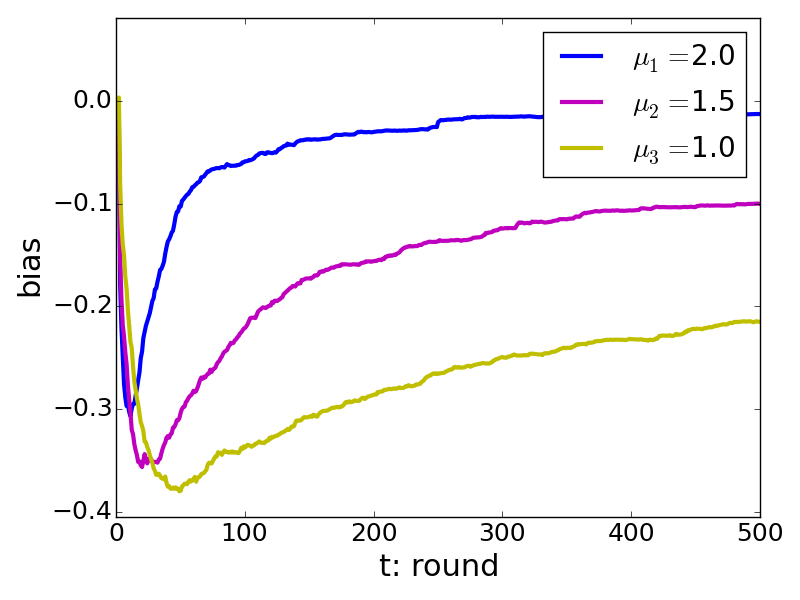}}\quad
  \subcaptionbox{lil' UCB, $\mu$ scale = 2}[.31\linewidth][c]{%
    \includegraphics[width=.31\linewidth]{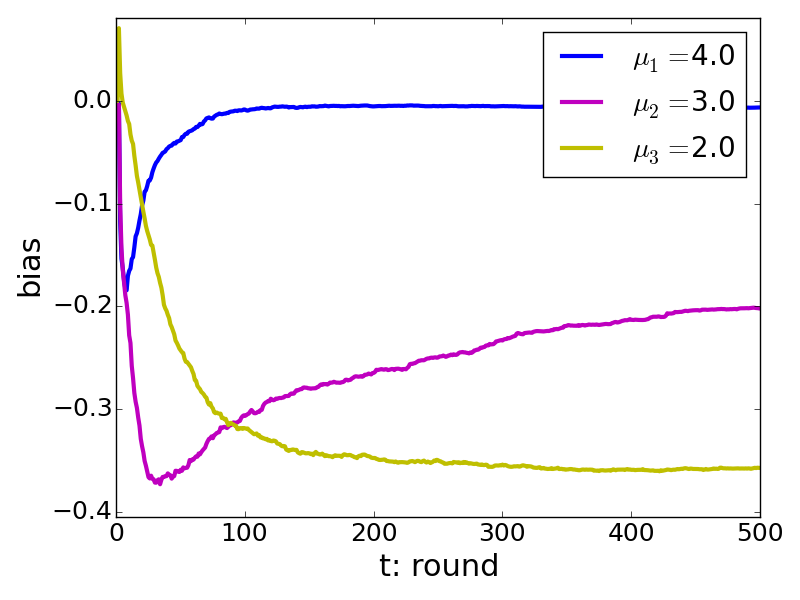}}\quad
  \subcaptionbox{lil' UCB, $\mu$ scale = 3}[.31\linewidth][c]{%
    \includegraphics[width=.31\linewidth]{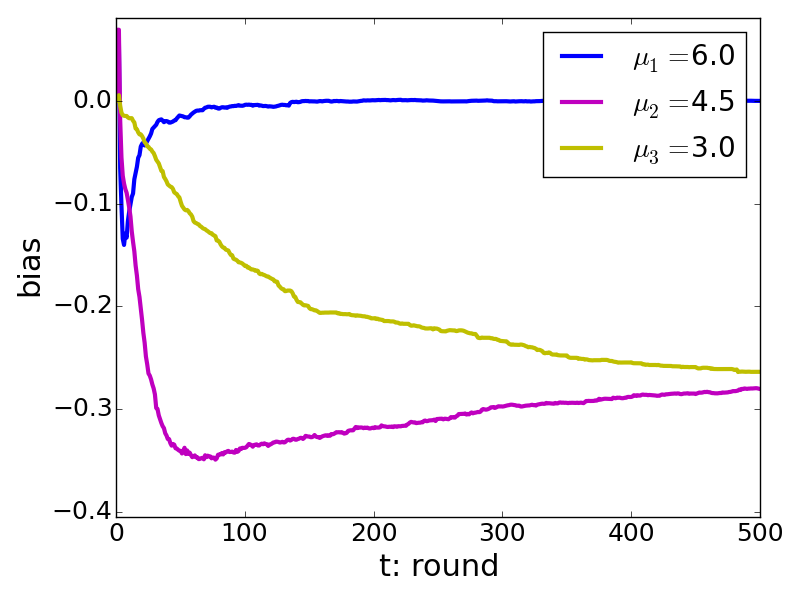}}

  \bigskip

  \subcaptionbox{Greedy, $\mu$ scale = 1}[.31\linewidth][c]{%
    \includegraphics[width=.31\linewidth]{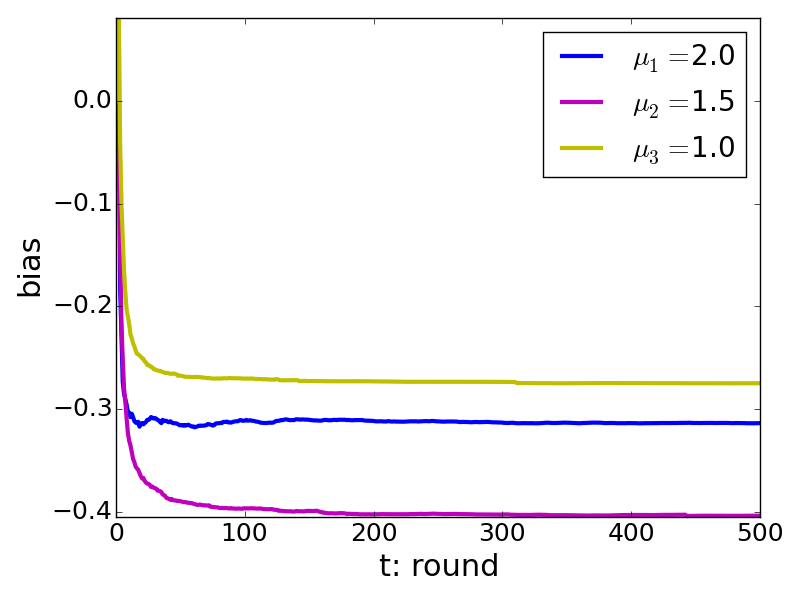}}\quad
     \subcaptionbox{\# future samples drawn given bias at t=100, with horizon T=1000}[.31\linewidth][c]{%
    \includegraphics[width=.31\linewidth]{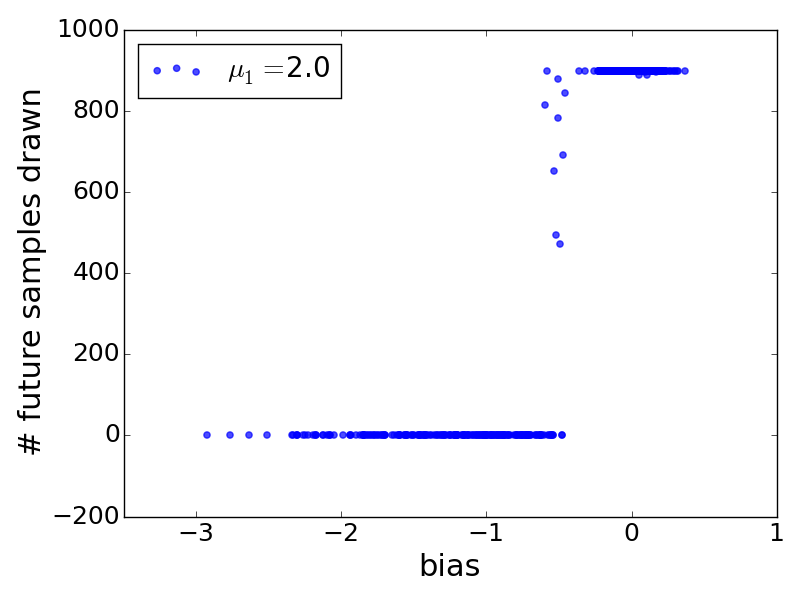}}
  \subcaptionbox{ cMLE debiasing}[.31\linewidth][c]{%
    \includegraphics[width=.31\linewidth]{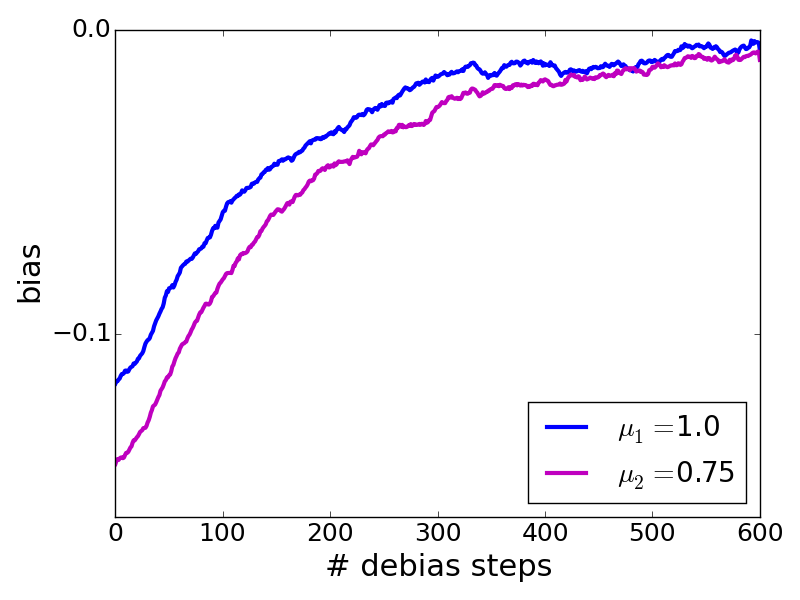}}\quad
 
\caption{In (a-c), we plot the bias of the empirical mean estimates of three unknown distributions running lil' UCB with horizon T=500. Each is distributed according to $\mathcal{N}(\mu_i,1)$, where $\mu_i$ is the mean of the $i$-th distribution, specified in the legends of the plot. We see that as we scale up $\mu_i$'s, so they become more spread out, the bias increases/decreases depending how far the $\mu_i$'s are from each other, and what is the order of the distributions. (d) plots the bias of the three unknown distributions running Greedy. (e) plots the number of future samples drawn from distribution 1 given its bias at t=100, running lil' UCB. Here T=1000 with two distributions, $\mathcal{N}(2,1)$ and $\mathcal{N}(1.5,1)$. This is a scatter plot over 1000 independent trials. (f) plots the bias as the estimate of the mean converges to the true mean across 600 gradient descent iterations. (a-d) and (f) are all averaged across 1000 independent trials.}
\label{fig:bias}
\end{figure*}
\paragraph{Experiments quantifying negative bias.}
We explore the effects on the bias from moving the distribution means apart. We used the lil' UCB algorithm, with algorithm specific parameters $\alpha=9,\beta=1,\epsilon=0.01, \delta=0.005$, which are the same as in the experiment section of \citet{jamieson2014lil}. We ran 1000 independent trials, with horizon $T=500$. We have three unknown distributions, all of the form $\mathcal{N}(\mu_i,1)$, with $\mu_1 = 2, \mu_2=1.5, \mu_3=1$. In this experiment, we scale the $\mu_i$'s by a scaling factor of $1, 2, 3$, and observe the bias of the empirical mean estimates of the three distributions. In Figure~\ref{fig:bias}(a) (b) (c), we plot the bias with the number of rounds. 

We first observe all distributions have negatively biased estimates of their true means. Further, these plots illustrate our intuition on the effect of \emph{consequential competition}. The distribution with the second best mean (the magenta curve) has worse bias as we scale up the $\mu_i$'s. We hypothesize when the distributions with the second best and the best means are close together, having one bad sample value for the second best one does not affect its chance of being sampled again as much as when their means are farther apart. On the other hand, for the distribution with the lowest true mean (the yellow curve), we observe its bias becomes worse first and then better as we scale up the $\mu_i$'s. We hypothesize that the same reason as before explains why the bias becomes worse first. However, as we further scale up the $\mu_i$'s, the bad sample values from the distribution with the lowest mean does not affect its future chances of being drawn much more than the good samples values, since its true mean is too far from the distribution with the highest mean. 

Next we compare lil' UCB with Greedy, see sub-figure $(a)$ and $(d)$ in Figure~\ref{fig:bias}. We observe that with Greedy in our setting, the empirical mean estimates for the distribution with the lowest mean has the least bias, followed by the distribution with the highest true mean. This is an example in which the distribution with the highest mean might not incur the least bias. With lil' UCB, the bias for the distribution with the highest true mean converges to 0 quickly, but with Greedy it plateaus. In lil' UCB, since it achieves optimal regret, the algorithm finds the distribution the highest true mean in finite number of time steps. The samples we get from that distribution become close to i.i.d. samples as $t$ increases, since the effect of the competition from other distributions is reduced over time. In Greedy it's known that the algorithm can be stuck on drawing from a suboptimal distribution, in which case the empirical average of the particular samples we have drawn from the distribution with the highest true mean must have a negative bias for this to happen. The bias of the best distribution thus doesn't converge to 0. Note that for both lil' UCB and Greedy, the suboptimal distributions can stay negatively biased for large $T$. However, the negative bias from running lil’ UCB is less severe in magnitude, because it uses confidence intervals to better address the issue that the empirical mean may be artificially small purely due to randomness.

Figure~\ref{fig:bias}(e) shows at round step $t=100$ with horizon $T=1000$, running lil' UCB with the same hyper-parameters in the same setting as in Figure~\ref{fig:bias}(a), we plot the number of future samples drawn from the distribution with the highest mean (i.e. $\mu=2.0$) vs. the bias from the empirical average of samples drawn so far from this distribution at time $t=100$. This confirms our intuition that large negative bias is correlated with fewer future chances of getting sampled. 

Our theoretical analysis of Theorem~\ref{thm:negative_bias} shows that the \emph{marginal} bias of each distribution is negative. The \emph{joint }bias across all the distributions---e.g. how likely is it that all the distributions have negative bias---is also an interesting question. We empirically investigated the frequency at which negative bias occurs across distributions in simulations. In Table~\ref{table:freq-bias}, we run Greedy, lil' UCB, $\epsilon$-Greedy ($\epsilon=0.1$), and Thompson Sampling (shown as "TS" in Table~\ref{table:freq-bias}) on 5 different distributions (see caption of Table~\ref{table:freq-bias} for details) for 10,000 trials, and record the fraction of trials at which any $m$ distributions all have negative bias at $T=100$, where $m=0,\cdots,5$. We observe that the results are highly skewed towards large values of $m$, suggesting that it is much more frequent that more distributions simultaneously have negative bias. We have also included additional results in a variety of settings in Appendix~\ref{app:joint-bias} that confirm this finding. Theoretical analysis of the joint bias across distributions is beyond the scope of this paper and is an interesting direction of future investigation. 

  \begin{table}[]
\centering
\caption{We run each of the following common bandit algorithms (note "TS" stands for Thompson Sampling) across 10,000 independent trials with 5 distributions, each with $\mathcal{N}(\mu_i, 1)$, where $\mu_1 = 1.0, \mu_2=0.75, \mu_3=0.5, \mu_4=0.38, \mu_5=0.25$. In each column, we record the fraction of trials in which $m$ distributions have negative bias at $T=100$, where $m=0,\cdots, 5$. In $\epsilon$-Greedy, $\epsilon=0.1$. In Thompson Sampling (TS), all distributions have prior $\mathcal{N}(0,25)$.}
\label{table:freq-bias}
\begin{tabular}{|c|c|c|c|c|c|c|}
\hline
\multirow{2}{*}{} & \multicolumn{6}{c|}{$m$:\ \# of distr. with bias $<0$} \\ \cline{2-7} 
                  & 0           & 1           & 2            & 3            & 4            & 5          \\ \hline
Greedy            & 0.02   &0.09  &0.23  &0.34  &0.24   &0.08         \\ \hline
lil' UCB           & 0.01   &0.05 & 0.21  &0.36  &0.30   &0.08        \\ \hline
$\epsilon$-Greedy &0.02  &0.12 & 0.27 & 0.33  &0.21  & 0.05         \\ \hline
TS  &0.01  & 0.08 & 0.24&  0.34 & 0.26  & 0.07        \\ \hline
\end{tabular}
\end{table}
\section{DEBIASING ALGORITHMS AND EXPERIMENTS}

\paragraph{Data splitting.}
A simple approach to obtain unbiased estimators of $\mu_k$'s is to split the data. 
Data splitting dates back to \citet{cox1975note} and has been discussed by in the context of identifying loci of interest in genetics \citep{sladek2007genome}, and online search advertising \citep{xu2013estimation}. \citet{wasserman2009high} and \citet{meinshausen2009p} discussed data splitting in high-dimensional inference. \citet{fithian2014optimal} discussed data splitting in post-selective inference. 
 Let $k$ be the distribution the selection function $f$ chooses at time $t$. Instead of taking one sample from $k$, we maintain a "held-out" set by taking an additional independent sample from $k$. We use the first samples as the sample history for $f$ which determines the future selections, and use the "held-out" set composed of the second samples for mean estimation. Since the "held-out" set is composed of i.i.d. samples that are independent of the selection process, its sample average is an unbiased estimate of $\mu_k$. However, if the total number of samples collected is fixed at $T$ rounds, then data splitting suffers from high variance, since half of all the samples are discarded in estimation. Data splitting is a natural baseline and we compare it with more sophisticated debiasing algorithms.

\paragraph{Conditional Maximum Likelihood Estimator (cMLE).} 
Data splitting is a general approach since it is agnostic to the selection function $f$. If we know the $f$ used to collect the data, then more powerful debiasing could be achieved by explicitly conditioning on the sequence of distributions that are selected by $f$ in a maximum likelihood framework. This conditioning approach is motivated by the recent successes of selective inference, which have been applied, for example, to debias the confidence intervals of the Lasso-selected features by conditioning on the Lasso algorithm \citep{taylor2015statistical}. To the best of our knowledge, our paper is the first extension of selective inference to adaptive data collection.    
To illustrate the cMLE approach, we consider the special case where  the decision on which distribution to sample at round $t$ is based on comparing the
decision statistics of the form, 
\begin{eqnarray}
\label{eq:decision_stat}
\mathbf{U}_t \overset{\Delta}{=} \left(U\left(\mean{X_t^{(1)}}, N_t^{(1)}\right), \dots, U\left(\mean{X_t^{(K)}}, N_t^{(K)}\right)\right).
\end{eqnarray}
$\mathbf{U}_t$ depends only on the empirical average $\mean{X_t^{(k)}}$'s and the number of samples $N_t^{(k)}$'s for $k\in[K]$. In other words, the selection function $f$ depends on
the history of rewards $\Lambda_t$ only through $\bfU_t$. In Greedy, 
$U\left(\mean{X_t^{(k)}}, N_t^{(k)}\right) = \mean{X_t^{(k)}}$, while in UCB type algorithms, $U_t^{(k)}$ will be the
upper confidence bounds that depend on both $\mean{X_t^{(k)}}$'s and $N_t^{(k)}$'s, where $U_t^{(k)}$
is shorthand for $U\left(\mean{X_t^{(k)}}, N_t^{(k)}\right)$. 

\begin{theorem}\label{thm:ll}
Let $s_t=f(\Lambda_t)$. Suppose the distributional function for distribution $k$ has density $h_{\theta^{(k)}}$, then
the conditional likelihood of the adaptive data collection problem is proportional to
\begin{eqnarray}\label{eq:ll}
p(\Lambda_T \mid s_t, ~t=1,\dots,T) \propto \prod_{k=1}^K \prod_{m=1}^{N_T^{(k)}} h_{\theta^{(k)}}(X_m^{(k)}) 
\notag
\\\cdot 
\prod_{t=K}^{T-1} \Parg{f(\mathbf{U}_t) = s_{t+1} \mid \mathbf{U}_t}.
\end{eqnarray}

To maximize the conditional likelihood, we need to solve the following optimization problem,
\begin{eqnarray}\label{eqn:cMLE_opt}
\max_{\theta} ~ \sum_{k=1}^K \sum_{m=1}^{N_T^{(k)}} \log\left[h_{\theta^{(k)}}(X_m^{(k)})\right]&
\notag
\\+ \sum_{t=K}^{T-1} \log\bigg[ \Parg{f(\mathbf{U}_t) = s_{t+1} \mid  \mathbf{U}_t}\bigg]&-\log Z(\theta),
\end{eqnarray}
where $\theta = (\theta^{(1)}, \dots, \theta^{(K)})$ are the parameters of interest and $Z(\theta)$
is the partition function in Eqn.~\eqref{eq:ll}, that only depends on the parameters $\theta$.

\end{theorem}
Theorem~\ref{thm:ll} gives an explicit form for the likelihood function
of the adaptive data collection problem (up to a constant). Note that in Eqn~\ref{eq:ll}, $\Lambda_T$ contains both $X_m^{(k)}$, and $\{\bm{U}_t\}_{t=1}^T$ (recall that $\Lambda_T$ is the history of samples up to time $T$). We give a proof of Theorem~\ref{thm:ll} in Appendix~\ref{app:proofs}. 

\paragraph{Adding additional noise to the sample values to improve cMLE optimization.}
We introduce
additional randomization when selecting a distribution. The reasons are two-fold. First, we need exponential-tailed noise in randomization to achieve asymptotically consistent estimates \citep{randomized_response,selective_bayesian}. Second, adding randomization smooths out the the hard boundaries in the sample space in evaluating $\Parg{f(\mathbf{U}_t) = s_{t+1} \mid \mathbf{U}_t}$. 
%Details see Appendix~\ref{app:hard-max-diff}. 
For example, in Greedy,
\begin{eqnarray}
\label{eq:hardmax}
\Parg{f(\mathbf{U}_t) = s_{t+1} \mid \mathbf{U}_t} = \indic{\argmax_k \mean{X_t^{(k)}} = s_{t+1}},
\end{eqnarray}
which means to compute the cMLE, we need to maximize the log-likelihood in a
constrained region of the sample space. 
%However, since the comparisons are made on the
%sample average $\mean{\bfX_t}=\left(\mean{X_t^{(1)}},\dots,\mean{X_t^{(K)}}\right)$, it induces a complicated constrained region
%on the sample history $\Lambda_T$, which we need to take expectation over in the partition function in Eqn.~\ref{eqn:cMLE_opt}. 
Optimization on such a region is no easy task.
Moreover, since the hard-max function induces discontinuities of likelihood along
the boundary of the constrained region, the cMLE will be ill-behaved since the gradient of the log-likelihood can become infinite \citep{randomized_response,selective_bayesian}. 

We propose adding Gumbel noise to the decision statistics $\mathbf{U}_t$ to smooth out $\Parg{f(\mathbf{U}_t) = s_{t+1} \mid \mathbf{U}_t}$. Note that we could also use other heavy-tailed distributions for the added noise. The Gumbel distribution offers computational convenience, since  
\begin{eqnarray}
\label{eq:softmax}
\Parg{f(\mathbf{U}_t) = k \mid \mathbf{U}_t} = \frac{\exp[U_t^{(k)} / \tau_t]}{\sum_{i=1}^K \exp[U_t^{(i)}/\tau_t]}
\end{eqnarray}
has a closed form due to the Gumbel-max trick \citep{gumbel1954statistical} (also see Lemma~\ref{lem:gumbel} in Appendix~\ref{app:gumbel}). 

We can now optimize Eqn.~\ref{eqn:cMLE_opt} using contrastive divergence \citep{contrastive_divergence}. Details of the algorithm are in Appendix~\ref{app:cmle-cd}. For lil' UCB or Greedy, we can compute $\mathbf{U}_t$ deterministically
from $\mean{\mathbf{X}_t}$ and $\mathbf{N}_t$. The selection function after Gumbel randomization is
defined as
\begin{eqnarray}
f(\mathbf{U}_t) = \argmax_{k} U_t^{(k)} + \epsilon_t^{(k)}, \quad \epsilon_t^{(k)} \overset{\textrm{iid}}{\sim} G_{\tau_t},
\end{eqnarray}
where $G_{\tau}$ is a Gumbel distribution of mean $0$ and scale parameter $\tau$. Similarly, we can also add Gumbel noise to $\epsilon$-Greedy to derive smooth conditional probabilities. We give examples of computing the conditional likelihood functions of common bandit algorithms with added Gumbel noise in Appendix~\ref{app:example-condi}. 

We summarize the debiasing procedure in Algorithm~\ref{alg:debias}. Note that we only compute cMLE with contrastive divergence (see Algorithm~\ref{alg:CD} in Appendix~\ref{app:cmle-cd}) once at time step $T$ when we wish to debias the estimates. Note that with these smooth $\Parg{f(\mathbf{U}_t) = k \mid \mathbf{U}_t}$ in Eqn.~\ref{eq:softmax}, we have well-behaved gradients in the parameter updates in computing cMLE.
\setlength{\textfloatsep}{5pt}
\begin{algorithm}
\begin{algorithmic}
\State {\bf Add Gumbel noise} when choosing which distribution to sample from at each time step $t$. Instead of applying the
selection function directly to $\bfU_t$, we apply it to 
$$
\left(U_t^{(k)} + \epsilon_t^{(k)}\right), \quad k=1,\dots,K
$$
where $\epsilon_t^{(k)} \overset{\textrm{iid}}{\sim} G_{\tau_t}$.
\State {\bf Compute conditional likelihood} by computing the selection probabilities,
$$
\Psubarg{\epsilon_t}{f(\bfU_t) = s_{t+1} \mid \bfU_t}.
$$
Note that here $f$ incorporates the randomness of Gumbel randomizations $\{\epsilon_t^{(k)}\}_{k\in[K]}$ as well as
the randomness in the original bandit algorithm.
\State {\bf Compute cMLE} using approximate gradient descent
with contrastive divergence. 
\end{algorithmic}
\caption{Debiasing algorithm using cMLE. Note that we only compute cMLE with contrastive divergence (see Algorithm \ref{alg:CD} in Appendix \ref{app:cmle-cd}) once at time step $T$ when we wish to debias the estimates.}
\label{alg:debias}
\end{algorithm}

\begin{table*}[ht]
\centering
\caption{\textbf{Bias Reduction.} With $K=2$, each distribution is drawn from $\mathcal{N}(\mu_i,1)$. where $\mu_1=1.0, \mu_2=0.75 $. With $K=5$, each distribution is drawn from $\mathcal{N}(\mu_i,1)$. where $\mu_1=1.0, \mu_2=0.75, \mu_3=0.5, \mu_4=0.38, \mu_5=0.25.$ In the left columns under each algorithm, we record the bias of the original algorithm at different time steps $T$. In the right columns, we record the percentage of the original bias that still remains after we run cMLE by adding gumbel noise $\epsilon_g \sim G_{\tau}$, with scale parameter $\tau = 1.0$, and contrastive divergence with 600 gradient descent iterations. All results are averaged across 1000 independent trials.}
\begin{tabular}{|c|c|c|c|c|c|c|}
\hline
\multirow{2}{*}{} & \multicolumn{2}{c|}{lil' UCB} & \multicolumn{2}{c|}{$\epsilon$-Greedy ($\epsilon=0.1$)} & \multicolumn{2}{c|}{Greedy} \\ \cline{2-7} 
                  & orig.        & cMLE        & orig.                      & cMLE                     & orig.        & cMLE       \\ \hline
T=8,K=2           & -0.26        & \textbf{6.2\%}         & -0.25                      & \textbf{7.3\%}                      & -0.29        & \textbf{2.8\%}        \\ \hline
T=16,K=2          & -0.29        & \textbf{5.2\%}         & -0.25                      & \textbf{1.6\%}                      & -0.32        & \textbf{8.3\%}        \\ \hline
T=20,K=5          & -0.32        & \textbf{14.9\%}        & -0.31                      & \textbf{9.1\%}                      & -0.35        & \textbf{18.0\%}       \\ \hline
T=40,K=5          & -0.35        & \textbf{14.2\%}        & -0.27                      & \textbf{8.8\%}                      & -0.37        & \textbf{15.9\%}       \\ \hline
\end{tabular}
\label{table:bias-full}
\end{table*}
\begin{table*}[ht]
\centering
\caption{\textbf{Mean Squared Error(MSE) reduction.} The experiment setup is the same as in Table \ref{table:bias-full}. The leftmost columns under each algorithm is the MSE of the original algorithm. The second to the left columns are the MSE percentage ratio of the data splitting with a held-out set compared to the MSE of the original algorithm. The right columns are the MSE percentage ratio of the cMLE algorithm after debiasing compared to the MSE of the original algorithm. For $\epsilon$-Greedy, we additionally run propensity matching (prop). Note that both data splitting and prop suffer from high variance despite achieving consistent estimation.}
\label{table:mse-full}
\begin{tabular}{|c|c|c|c|c|c|c|c|c|c|c|}
\hline
\multirow{2}{*}{} & \multicolumn{3}{c|}{lil' UCB}  & \multicolumn{4}{c|}{$\epsilon-$Greedy($\epsilon=0.1$)} & \multicolumn{3}{c|}{Greedy}  \\ \cline{2-11} 
                  & orig. & held  & cMLE        & orig.          & held     &prop     & cMLE                & orig & held  & cMLE        \\ \hline
T=8,K=2           & 0.56  & 108\% & \textbf{86\%} & 0.51           & 123\%  & 295\%        & \textbf{76\%}         & 0.56 & 108\% & \textbf{78\%} \\ \hline
T=16,K=2          & 0.50  & 101\% & \textbf{40\%} & 0.38           & 123\%   & 244\%       & \textbf{52\%}         & 0.53 & 107\% & \textbf{45\%} \\ \hline
T=20,K=5          & 0.57  & 112\% & \textbf{99\%} & 0.52           & 123\%   & 399\%         & \textbf{94\%}         & 0.59 & 111\% & \textbf{89\%} \\ \hline
T=40,K=5          & 0.54  & 104\% & \textbf{52\%} & 0.39           & 135\%   & 290\%       & \textbf{62\%}         & 0.54 & 107\% & \textbf{52\%} \\ \hline
\end{tabular}
\end{table*}

\paragraph{Debiasing experiments.}
We empirically show that the cMLE algorithm can reduce bias significantly and reduce the mean squared error (MSE) as well. In Table~\ref{table:bias-full}, we see significant bias reduction for the lil' UCB, $\epsilon$-Greedy, and Greedy algorithms using the cMLE debiasing algorithm, in both the $K=2$ and $K=5$ cases, where $K$ is the number of distributions. All of our experiments used the same implementation of cMLE with the same hyper-parameters to demonstrate that cMLE can be robustly applied to different distribution settings without fine-tuning.  We could still have some residual bias after running cMLE. This is due to the guarantee of asymptotic consistency with added heavy-tailed noise (i.e. the bias tends to 0 as $T$ tends to infinity), and $T$ is finite in our experiments. Table~\ref{table:mse-full} shows the reduction of MSE. The data splitting algorithm achieves consistent estimates, but it incurs high variance since the effective sample size is halved by maintaining a held-out set. Empirically we observe that data splitting suffers from high MSE. 
In Figure~\ref{fig:bias}(f), we run Greedy with cMLE, with two distributions, $\mathcal{N}(1,1)$ and $\mathcal{N}(0.75,1)$. We show the convergence of the estimated mean to the true mean as we run gradient descent over 600 iterations. 
%All experiments use gradient descent learning rate $\eta=0.01$, 30 steps of MCMC (with the first half of the steps as burn-in), and dynamically adjusted step size of MCMC to ensure the acceptance ratio is between $20\%-50\%$. 
We see that cMLE significantly reduces the bias, while improving the MSE.  We also experimented with propensity matching, a commonly used method that weights each observed value of a distribution by one over the probability that this distribution is selected \citep{austin2011introduction}. Propensity matching  is unbiased, but has very large variance and thus a much greater MSE by several fold compared to cMLE. We discuss it in more detail in Appendix~\ref{app:prop}. 

In both Table~\ref{table:bias-full} and Table~\ref{table:mse-full}, we looked at the cases where the horizon $T$ is relatively small. The reasons are two-fold. First, a relatively small $T$ is relevant in many biomedical settings, where the scientist adaptively collects data from several arms corresponding to different experimental conditions, and such collection procedures are expensive. Second, empirically, the magnitude of the bias tends to be the largest when $t$ is small (c.f.Figure~\ref{fig:bias} (a)-(d)), so we focused on cases where the bias is the largest to demonstrate the effectiveness of cMLE to debias the empirical estimates. 
For $\epsilon$-Greedy, for example, the biases of all the distributions are essentially 0 when $T$ is large, and cMLE is not needed in this regime.  
For the Greedy algorithm, the bias of the suboptimal distributions are stuck at quite negative values even for very large $T$ because they are not sampled again after the first few rounds. For completeness, we have also included the debiasing results for the Greedy algorithm when the horizon is large ($T=1000$) in Table~\ref{table:greedy-t-1000}. We observe here that cMLE can almost completely debias the Greedy algorithm. This is expected since cMLE is asymptotically consistent, so as $T$ grows, the bias reduces to 0 significantly. The mean squared error (MSE) has also reduced to a negligible amount for cMLE, but remains huge for the data splitting method. We have also included results running Thompson Sampling in Appendix~\ref{app:thompson} for completeness. 
\begin{table}[!htb]
\centering
\caption{The bias and mean squared error (MSE) of running the Greedy algorithm with $T=1000$. The experiment setup is the same as in Table~\ref{table:bias-full}, for $K=2$ and $K=5$, where $K$ is the number of distributions of interest. The columns under "Bias" record the bias incurred by the original algorithm under "orig.", and the percentage of bias remaining after running the cMLE algorithm under "cMLE". The columns under "MSE" record the mean squared error (MSE) under the original greedy algorithm, the percentage of MSE running the data splitting method (i.e. using a held-out set of samples) in comparison to the original MSE, and the percentage of MSE running cMLE in comparison to the original MSE. We observe the cMLE has superior performance in both bias and MSE reduction.}
\label{table:greedy-t-1000}
\begin{tabular}{|c|c|c|c|c|c|}
\hline
\multirow{2}{*}{} & \multicolumn{2}{c|}{Bias} & \multicolumn{3}{c|}{MSE} \\ \cline{2-6} 
& orig.        & cMLE       & orig.  & held-out & cMLE \\ \hline
K=2 & -0.2 & \textbf{0.0\%} & 0.255 & 89.8\% & \textbf{0.4\%}\\ \hline
K=5 & -0.21 & \textbf{1.0\%} & 0.277 & 94.9\%& \textbf{1.1\%}\\ \hline
\end{tabular}
\end{table}
\section{DISCUSSION}
Our main result shows that adaptively collected data are negatively biased when the data collection algorithm $f$ satisfies \emph{Exploit} and \emph{IIO}. This seems counterintuitive because we typically associate optimization (as in exploitative algorithms) with a positive selection bias (i.e.Winner's Curse). For example, if we draw 10 samples from $\mathcal{N}(0,1)$ and report the $\max$, then we have positive reporting bias. The reason for the discrepancy between these phenomena is that for any sample history of data, the ``best'' option $k$'s sample mean is likely to be larger than its true mean. However who is the ``best'' varies in different sample paths, and the bias of each distribution $k$ is negative in expectation. 

We explored data splitting and cMLE as two approaches to reduce this bias. Data splitting is unbiased but suffers larger MSE because it ignores half of the samples during estimation. cMLE can reduce bias close to 0 while also reducing MSE. The trade-off is that it requires specific knowledge about $f$ and also requires one to add additional noise to the collected data. Both approaches require modifying the data collection procedure and cannot be generically applied to debias existing adaptively collected data.
%\hl{which is the reason why we cannot test the effectiveness of cMLE on non-synthetic datasets which are observational (i.e. the collection procedures of which cannot be modified).} 
As adaptively collected data are ubiquitous, developing flexible debiasing approaches to debias observational data is an important direction of future research.    

\newpage
\bibliographystyle{abbrvnat}
\bibliography{references}
\nocite{*}

\newpage
\onecolumn
\appendix
\section{lil' UCB Algorithm} \label{app:lil' UCB}
lil' UCB Algorithm is proposed by \citet{jamieson2014lil}, and achieves optimal regret. It has become one of the most popular upper confidence bound type algorithms. 

In lil' UCB, the selection function 
\begin{align}
f(\Lambda_t) = \argmax_k \mean{X_t^{(k)}} + (1+\beta)(1+\sqrt{\epsilon}) \sqrt{\frac{2(1+\epsilon)\log (\frac{\log((1+\epsilon)N_t^{(k)})}{\delta})}{N_t^{(k)}}}
\end{align}
where $N_t^{(k)}$ is the number of times arm $k$ gets pulled by time $t$, and $\mean{X_t^{(k)}} \equiv \frac{\sum_{i=1}^{N_t^{(k)}} X_i^{(k)}}{N_t^{(k)}}$.
$\epsilon, \delta, \beta$ are lil' UCB hyper-parameters as specified in \citet{jamieson2014lil}.

\section{Proofs of the main results}\label{app:proofs}
\begin{proof}[Proof of Theorem~\ref{thm:negative_bias}]
Without loss of generality, we focus on showing that distribution 1 has negative bias. The argument applies directly to every other distribution. 
For a given history $\Lambda_t$, $f(\Lambda_t)$ is a random variable over $[K]$.   
We define two independent random variables based on $f(\Lambda_t)$. Let $g(\Lambda_t) =\indic{f(\Lambda_t)=1}$. Let $h\left(\Lambda_t^{(-1)}\right) = f(\Lambda_t) | f(\Lambda_t)\neq 1$ be a random variable with support $\{2,...,K\}$, such that for $k\in \{2, ..., K\}$,
\[
 \Pr\left[h\left(\Lambda_t^{(-1)}\right) = k\right] = \Pr[f(\Lambda_t) = k | f(\Lambda_t) \neq 1] = \frac{\Pr[f(\Lambda_t) = k]}{\sum_{i=2}^K  \Pr[f(\Lambda_t) = i]}.
 \]
  Note that $f$ satisfies \emph{IIO} implies that the law of $h$ is only a function of $\Lambda_t^{(-1)}$, which is the history only of the distributions $2, ..., K$ up to time $t$. It's clear that distribution selection by $s_{t+1} = f(\Lambda_t)$ is equivalent to  (i.e. have the same law as) 
\begin{eqnarray}
  s_{t+1}=\begin{cases}
    1, & \text{if $g(\Lambda_t)=1$}.\\
    k, & \text{if $g(\Lambda_t)=0$, $h(\Lambda_t^{(-1)})=k$, $k\in[2,K]$}.
  \end{cases}
  \end{eqnarray}
Since this equivalence holds for every $t$, the adaptive data collection procedure is defined by the independent random variables $g(\Lambda_t)$ and $h(\Lambda_t^{(-1)})$. 

To study distribution 1 we condition on the realization $\Theta$, where $\Theta$ includes the realizations of distributions $k$ for $k \in \{2, ..., K\}$ and $T$ random seeds for $g$ and $h$, $\{\omega_{g,t}, \omega_{h,t}\}_{t=1}^{T}$. More precisely, $\Theta=\{\{x_t^{(k)}\}_{t=1}^T, \{\omega_{g,t}, \omega_{h,t}\}_{t=1}^T, k\in[2,K]\}$, where $x_t^{(k)}$ is a realized value of a sample drawn from distribution $k$ at round $t$. Then given any realization of distribution 1, $\bm{\sigma} = (\sigma_1,\sigma_2,\dots,\sigma_T)$, $\sigma_i \in \mathbb{R}$, conditioning on $\Theta$ induces a deterministic mapping $S(\bm{\sigma}) = (t_1, ..., t_T)$, where $t_i$ is a positive integer corresponding to the time when the $i$-th sampling of distribution 1 occurs. Note that $t_i \in [T]\bigcup \{*\}$, where $t_i = *$ indicates that the $i$-th drawing occurs after time $T$. 
Since all the other distributions' realization and randomness are fixed, $t_i$ is a deterministic function of $(\sigma_1, ..., \sigma_{i-1})$.  
 
Let $\tilde{t}_j$ indicate the round at which distribution 1 is \emph{not} selected for the $j$-th time, then  IIO implies $s_{\tilde{t}_j} = h(\Lambda_{\tilde{t}_j-1}^{(-1)}, \omega_{h,j})$. Which distribution among $2, \dots, K$ is selected is determined by $\Lambda_{\tilde{t}_j-1}^{(-1)}$, which is the history of distributions $2,\dots, K$ up to time $\tilde{t}_j-1$. Note that $s_{\tilde{t}_j}$ is a function of $\omega_{h,j}$ not $\omega_{h,\tilde{t}_j}$; i.e. the random seeds $\omega_{h,j}$ is only used when distribution 1 is not selected. From this observation, we see an important property of conditioning on $\Theta$.

\paragraph{Property 1.} If $\tilde{t}_j$ indicates the round at which distribution 1 is \emph{not} selected for the $j$-th time, then the history $\Lambda_{\tilde{t}_j}^{(-1)}$ is completely determined by the index $j$.

Our goal is to show that for an arbitrary realization $\Theta$, $\E \left[\mean{X^{(1)}_T} | \Theta \right] \leq \mu_1$. Then it would follow that  $\E \left[\mean{X^{(1)}_T}  \right] \leq \mu_1$. As we discussed above, after conditioning on $\Theta$, the data collection procedure is equivalent to a mapping $S(\bm{\sigma}) = (t_1,...,t_T)$. For a given path $\bm{\sigma} = (\sigma_1,...,\sigma_T)$, let $n_{\bm{\sigma}} = |\{t_i: t_i \leq T\}|$ be the number of times distribution 1 is selected by round $T$.  $S$ depends on $\Theta$, but we will not write this explicitly to simplify notation. Moreover,  $\Pr[\bm{\sigma} | \Theta] = \Pr[\bm{\sigma}] $ since the values of distribution 1 is independent of the realizations of the other distributions and the randomness in the selections. Therefore, 
\[
\E \left[\mean{X^{(1)}_T} | \Theta \right] = \sum_{\bm{\sigma}} \Pr[\bm{\sigma}]\frac{ \sum_{i=1}^{n_{\bm{\sigma}}} \sigma_i}{n_{\bm{\sigma}}}.
\]

Our proof strategy is to show that any mapping $S$ from paths $\bm{\sigma}$ to sets of times $(t_1, ..., t_T)$ which satisfies \emph{Exploit} condition must have bias $\leq$ 0. It suffices to consider the mapping $S$ corresponding to the largest $\E \left[\mean{X^{(1)}_T} | \Theta \right]$ and still satisfies \emph{Exploit}. We show that such a mapping $S$ must have the property that $n_{\bm{\sigma}}$ is the same constant for all path $\bm{\sigma}$. For such an $S$, it is immediate that $\E \left[\mean{X^{(1)}_T} | \Theta \right] = \mu_1$. 

Suppose for a maximal mapping $S$, $n_{\bm{\sigma}}$ differs for different $\bm{\sigma}$. Let $l$ be the largest integer for which there exist two paths $\bm{\sigma}$ and $\bm{\sigma}'$ such that $\sigma_i = \sigma'_i $ for $i < l$ and $n_{\bm{\sigma}} \neq n_{\bm{\sigma}'}$. So $\bm{\sigma}$ and $\bm{\sigma}'$ agree up to the $l-1$st drawing of distribution 1. We denote $\alpha \equiv \sigma_l$ and   $\alpha' \equiv \sigma'_l$; without loss of generality we can assume $\alpha < \alpha'$. 

\paragraph{Property 2.} The fact that $l$ is the largest such index implies that if $\bm{\sigma}''$ is any other path such that $\bm{\sigma}''_i = \sigma_i$ for $i \leq l$ then $n_{\bm{\sigma}''} = n_{\bm{\sigma}}$. Similarly if $\sigma''_i = \sigma'_i$ for $i \leq l$ then $n_{\bm{\sigma}''} = n_{\bm{\sigma}'}$.

There are two possible cases and we show that they both lead to contradictions. This would complete the proof by contradiction. 

\paragraph{Case 1: $n_{\bm{\sigma}} > n_{\bm{\sigma}'}$.} Consider the two paths $\bm{\lambda} = (\sigma_1,...,\sigma_{l-1},\alpha,\lambda_{l+1},...,\lambda_{T})$ and $\bm{\lambda'} = (\sigma_1,...,\sigma_{l-1},\alpha',\lambda_{l+1},...,\lambda_{T})$, where $\lambda_{l+1}...\lambda_{T}$ is some arbitrary fixed string of realizations. Property 2 implies that $n_{\bm{\lambda}} = n_{\bm{\sigma}} > n_{\bm{\sigma}'} = n_{\bm{\lambda}'}$. Under the mapping $S$, $\bm{\lambda}$ and $\bm{\lambda}'$ maps onto two sets of times $\{t_{\lambda,i}\}_{i=1}^T$ and $\{t_{\lambda',i}\}_{i=1}^T$, where $t_{\lambda,i}$ (resp. $t_{\lambda',i}$) is the round at which distribution 1 is drawn the $i$-th time under the realization $\bm{\lambda}$ (resp. $\bm{\lambda}'$). Since at least the first $l-1$ terms of $\bm{\lambda}$ and $\bm{\lambda}'$ are equal, at least the first $l$ terms of $t_{\bm{\lambda},i}$ and $t_{\bm{\lambda}',i}$ are equal since the $k$-th term of $t_{\bm{\lambda},i}$ depends on the first $k-1$ terms of $\bm{\lambda}$ for all $0 < k \leq T$.   Let $l_1 > l$ be the first index where $t_{\bm{\lambda},l_1}<t_{\bm{\lambda}',l_1}$. There must exist such a $l_1$ in order for  $n_{\bm{\lambda}} > n_{\bm{\lambda}'}$. 

Consider the round $t^* = t_{\bm{\lambda}, l_1} -1$. The histories up to round $t^*$ of paths $\bm{\lambda}$ and $\bm{\lambda}'$, i.e. $\Lambda^{(-1)}_{\bm{\lambda}, t^*}$ and $\Lambda^{(-1)}_{\bm{\lambda}', t^*}$, are identical because in both paths distribution 1 has been selected $l_1-1$ times by round $t^*$ (by Property 1). Moreover the empirical average of distribution 1 under $\bm{\lambda}$ is strictly lower than the average under $\bm{\lambda}'$. \emph{Exploit} property states that $g(\Lambda_{\bm{\lambda}, t^*}, \omega_{g, t^*}) = 1= f(\Lambda_{\bm{\lambda}, t^*}, \omega_{g, t^*}) $ implies $f(\Lambda_{\bm{\lambda}', t^*}, \omega_{g, t^*}) = 1 = g(\Lambda_{\bm{\lambda}', t^*}, \omega_{g, t^*})$. This implies that $t_{\bm{\lambda},l_1}=t_{\bm{\lambda}',l_1}$, contradicting $t_{\bm{\lambda},l_1}<t_{\bm{\lambda}',l_1}$. Therefore the scenario $n_{\bm{\sigma}} > n_{\bm{\sigma}'}$ is not possible if $f$ satisfies \emph{Exploit}. Note that for any $\Lambda_t$,  we can use the same probability space $\Omega$ for $g(\Lambda_t)$ and $f(\Lambda_t)$ such that  $\{\omega:  g(\Lambda_t, \omega) = 1 \} = \{\omega:  f(\Lambda_t, \omega) = 1 \}$.
 
 \paragraph{Case 2: $n_{\bm{\sigma}} < n_{\bm{\sigma}'}$.} By Property 2,  all the path where the first $l$ terms are $\sigma_1...\sigma_{l-1}\alpha$ have $n_{\bm{\sigma}}$ total number of draws. The contribution of these paths to the average $\mean{X^{(1)}_T} $ is 
\[ 
  \E\left[\mean{X^{(1)}_T}|\Theta, \sigma_1,...,\sigma_{l-1},\alpha \right] = \frac{\sum_{i=1}^{l-1}\sigma_i+\alpha+(n_{\bm{\sigma}} - l)\mu_1 }{n_{\bm{\sigma}}}.
  \] 
  
  Similarly, all the path where the first $l$ terms are $\sigma_1...\sigma_{l-1}\alpha'$ have $n_{\bm{\sigma}'}$ total number of draws. The contribution of these paths to the average $\mean{X^{(1)}_T} $ is 
\[ 
  \E\left[\mean{X^{(1)}_T}|\Theta, \sigma_1,...,\sigma_{l-1},\alpha' \right] = \frac{\sum_{i=1}^{l-1}\sigma_i+\alpha'+(n_{\bm{\sigma}'} - l)\mu_1 }{n_{\bm{\sigma}'}}.
  \] 
Since $\frac{\sum_{i=1}^{l-1} \sigma_i + \alpha}{l} < \frac{\sum_{i=1}^{l-1} \sigma_i + \alpha'}{l}$, we must have either of the following hold: 
 \begin{enumerate}
 \item $\frac{\sum_{i=1}^{l-1} \sigma_i + \alpha}{l} < \mu_1$. If this holds true, then the paths where the first $l$ terms are $\sigma_1...\sigma_{l-1}\alpha$ can have $m$ instead of $n_{\bm{\sigma}}$ total number of draws, where $n_{\bm{\sigma}} < m \leq n_{\bm{\sigma}'}$. Note that $\frac{\sum_{i=1}^{l-1}\sigma_i+\alpha+(n_{\bm{\sigma}} - l)\mu_1 }{n_{\bm{\sigma}}} < \frac{\sum_{i=1}^{l-1}\sigma_i+\alpha+(m - l)\mu_1 }{m}$. This modification preserves \emph{Exploit} property while increasing $\E\left[\mean{X^{(1)}_T}|\Theta, \sigma_1,...,\sigma_{l-1},\alpha \right]$, and thus increasing the $\E\left[\mean{X^{(1)}_T}|\Theta \right] $ of $S$. This contradicts the assumption that $S$ is the maximal mapping. 
 
 \item $\frac{\sum_{i=1}^{l-1} \sigma_i + \alpha'}{l} > \mu_1$. If this holds true, then the paths where the first $l$ terms are $\sigma_1...\sigma_{l-1}\alpha'$ can have $m'$ instead of $n_{\bm{\sigma}'}$ total number of draws, where $n_{\bm{\sigma}} \leq m < n_{\bm{\sigma}'}$. Note that $\frac{\sum_{i=1}^{l-1}\sigma_i+\alpha'+(n_{\bm{\sigma}'} - l)\mu_1 }{n_{\bm{\sigma}'}} < \frac{\sum_{i=1}^{l-1}\sigma_i+\alpha'+(m - l)\mu_1 }{m}$. This modification preserves \emph{Exploit} property while increasing $\E\left[\mean{X^{(1)}_T}|\Theta, \sigma_1,...,\sigma_{l-1},\alpha' \right]$, and thus increasing the $\E\left[\mean{X^{(1)}_T}|\Theta \right] $ of $S$. This contradicts the assumption that $S$ is the maximal mapping. 
 \end{enumerate}
 
The case analysis proves that in order for $S$ to be the mapping corresponding to the maximal $\E\left[\mean{X^{(1)}_T} | \Theta \right] $ it must assign the same  constant $n_{\bm{\sigma}}$ for all path $\bm{\sigma}$, i.e. the number of times distribution 1 is selected does not depend on its observed values.  Such a mapping is unbiased: $\E\left[\mean{X^{(1)}_T} | \Theta \right] = \mu_1$.
\end{proof}

\begin{proof}[Proof of Proposition.~\ref{prop:algs_satifies}]
For any algorithm with  the following form of the selection function,
\begin{eqnarray}
\label{eq:prop1_proof}
f\left(\Lambda_t^{(k)} \cup \Lambda_t^{(-k)}\right) = \argmax_{k\in[K]} U_t^{(k)}\left(\mean{X_t^{(k)}}, N_{t}^{(k)},\omega\right),
\end{eqnarray}
such that conditioning on $\Lambda_t^{(k)}$ and $\Lambda_t^{'(k)}$ with $N_t^{(k)}=N_t^{'(k)}$, and $\mean{X_t^{(k)}} < \mean{X_t^{'(k)}}$, and fixing $\Lambda_t^{(-k)}$ and $\omega$, we have $U_t^{(k)}(\mean{X_t^{(k)}}, N_{t}^{(k)},\omega) < U_t^{'(k)}(\mean{X_t^{'(k)}}, N_{t}^{'(k)},\omega)$, then it satisfies Exploit by definition. We show lil' UCB, Greedy, and $\epsilon$-Greedy can all be written in the form of Eqn.~\ref{eq:prop1_proof}. 

In lil' UCB,   
\begin{eqnarray}
U_t^{(k)}\left(\mean{X_t^{(k)}}, N_{t}^{(k)}, \omega\right) =U_t^{(k)}\left(\mean{X_t^{(k)}}, N_{t}^{(k)}\right) =  \mean{X_t^{(k)}} + (1+\beta)(1+\sqrt{\epsilon}) \sqrt{\frac{2(1+\epsilon)\log (\frac{\log((1+\epsilon)N_{t}^{(k)})}{\delta})}{N_t^{(k)}}}
\end{eqnarray}
where $\epsilon, \delta, \beta$ are lil' UCB hyper-parameters as specified in \citet{jamieson2014lil}. 
In Greedy, 
\begin{eqnarray}
U_t^{(k)}(\mean{X_t^{(k)}}, N_{t}^{(k)}, \omega) =U_t^{(k)}(\mean{X_t^{(k)}})= \mean{X_t^{(k)}}
\end{eqnarray}
In $\epsilon$-Greedy, 
\begin{eqnarray}
U_t^{(k)}(\mean{X_t^{(k)}}, N_{t}^{(k)}, \omega) =\begin{cases}
\mean{X_t^{(k)}}, & \text{ if } \omega >\epsilon\\
- & \text{ if } \omega< \epsilon
\end{cases}
\label{prop1:egreedy}
\end{eqnarray}

In Eqn.~\ref{prop1:egreedy}, when $\omega < \epsilon$, since we condition on $\omega$, it is trivially true that $f(\Lambda_t^{(k)} \cup \Lambda_t^{(-k)}) = k$ implies $f(\Lambda_t^{'(k)} \cup \Lambda_t^{(-k)}) = k$. In all of the above algorithms, $U_t^{(k)}$ monotonically increases as $\mean{X_t^{(k)}}$ increases, conditioning on $\omega$ and $N_t{(k)}$ fixed. Thus all three algorithms satisfy Exploit.

lil' UCB and greedy trivially satisfy IIO because they are deterministic algorithms. For $\epsilon$-Greedy, conditioning on $f(\Lambda_t) \neq k$ and $f(\Lambda_t) \neq k$, and $\Lambda_t^{(-k)}$, if $\omega > \epsilon$, then $f(\Lambda_t,\omega)$ is determined by $\Lambda_t^{(-k)}$. If $\omega <\epsilon$, then all the $K-1$ distributions are uniformly chosen in both cases.  
\end{proof}
\begin{proof}[Proof of Proposition.~\ref{prop:K=2}]
Without loss of generality, we focus on showing that distribution 1 has negative bias. We modify the arguments used to prove Theorem~\ref{thm:negative_bias}. To study distribution 1 we condition on the realization $\Theta$, where $\Theta$ includes the realization of distribution 2 and $T$ random seeds for $f$, $\{\omega_{t}\}_{t=1}^{T}$.  Then given any realization of distribution 1, $\sigma = (\sigma_1,\sigma_2,...,\sigma_T)$, $\sigma_i \in \mathbb{R}$, conditioning on $\Theta$ induces a deterministic mapping $S(\sigma) = \{t_1, ..., t_T\}$, where $t_i$ is a positive integer corresponding to the time when the $i$-th drawing of distribution 1 occurs. Note that $t_i \in [T]\cup {*}$, where $t_i = *$ indicates that the $i$-th drawing occurs after time $T$. 
Since the realizations of distribution 2 and the randomness in $f$ are fixed, $t_i$ is a deterministic function of $\{\sigma_1, ..., \sigma_{i-1}\}$. We also have the following property as a consequence.   

\paragraph{Property 1.} If $\tilde{t}_j$ indicate the $j$-th time where distribution 2 is selected, then the history $\Lambda_{\tilde{t}_j}^{(2)}$ is completely determined by the index $j$. 

The rest of the proof is identical to the proof of Theorem~\ref{thm:negative_bias}.
\end{proof}

\begin{proof}[Proof of Theorem~\ref{thm:ll}]
The conditional
likelihood $p_{\Lambda_T}(\Lambda_T \mid s_t, ~t=1, \dots, T)$ is related to the original likelihood $h_{\Lambda_T}(\Lambda_T) = \prod_{k=1}^K \prod_{m=1}^{N_T^{(k)}} h_{\theta^{(k)}}(X_m^{(k)})$ via the {\em selective likelihood
ratio (LR)} . 
\begin{eqnarray}
\label{eq:selective:lr}
LR(\bfU \mid s_t, t=1, \dots, T) \propto \prod_{t=K}^{T-1} \Parg{f(\mathbf{U}_t) = s_{t+1} \mid \mathbf{U}_t},
\end{eqnarray}
where $\bfU = (\bfU_t)_{t=1}^T$. The index starts from $K$ because we
always draw samples from each distribution once in the beginning.
The probability is taken over the extra randomness in the selection function $f$, fixing
the decision statistics $\bfU_t$'s and the sequence of choices $s_t$'s. Moreover, note that conditioning on the sequence of distribution to select $s_t$'s
means we are also fixing $\mathbf{N}_t$'s as they are equivalent.

Using the change of variable formula and the selective likelihood ratio in Eqn.~\ref{eq:selective:lr}, we have 
$$
\begin{aligned}
&p_{\Lambda_T}(\Lambda_T \mid s_t, ~t=1, \dots, T) \\
=& p_{\bfU}(\bfU \mid s_t, ~t=1, \dots, T) \times |\det \mathbf{J_{\Lambda_T \to U}}| \\
=& h_{\bfU}(\bfU) LR(\bfU \mid s_t, ~t=1, \dots, T) \times |\det \mathbf{J_{\Lambda_T \to U}}| \\
=& h_{\Lambda_T}(\Lambda_T) \times |\det \mathbf{J_{U \to \Lambda_T}}| \times LR(\bfU \mid s_t, ~t=1, \dots, T) \times 
|\det \mathbf{J_{\Lambda_T \to U}}| \\
\propto& h_{\Lambda_T}(\Lambda_T)\times \prod_{t=K}^{T-1} \Parg{f(\bfU_t) = s_{t+1} \mid \bfU_t},
\end{aligned}
$$
where $\mathbf{J_{\Lambda_T \to U}}$ is the Jacobian matrix for the map from $\Lambda_T \to \bfU$.
$h_{\Lambda_T}(\Lambda_T)$ is the unconditional likelihood of the data generating distribution.
Note the last equation is due to that there is an invertible (linear) map between $\Lambda_T$ and $\bfU$.
\end{proof}

\section{Additional experiment results for joint bias characterization}
\label{app:joint-bias}

  \begin{table}[!htbp]
\centering
\caption{We run each of the following common bandit algorithms (note "TS" stands for Thompson Sampling) across 10,000 independent trials with 2 distributions, each with $\mathcal{N}(\mu_i, 1)$, where $\mu_1 = 1.0, \mu_2=0.75$. In each column, we record the fraction of trials in which $m$ distributions have negative bias at $T=100$, where $m=0,1,2$. In $\epsilon$-Greedy, $\epsilon=0.1$. In Thompson Sampling (TS), all distributions have prior $\mathcal{N}(0,25)$.}
\label{table:freq-bias-2}
\begin{tabular}{|c|c|c|c|}
\hline
 & $m=0$           & $m=1$           & $m=2$                   \\ \hline
Greedy            &0.11  &0.50  &0.39        \\ \hline
lil' UCB           &0.20  &0.51  &0.29      \\ \hline
$\epsilon$-Greedy &0.20  &0.51  &0.29       \\ \hline
TS  &0.17  &0.51  &0.32      \\ \hline
\end{tabular}
\end{table}
 \begin{table}[!htbp]
\centering
\caption{We run each of the following common bandit algorithms (note "TS" stands for Thompson Sampling) across 10,000 independent trials with 3 distributions, each with $\mathcal{N}(\mu_i, 1)$, where $\mu_1 = 1.0, \mu_2=0.75, \mu_3=0.5$. In each column, we record the fraction of trials in which $m$ distributions have negative bias at $T=100$, where $m=0,1,2, 3$. In $\epsilon$-Greedy, $\epsilon=0.1$. In Thompson Sampling (TS), all distributions have prior $\mathcal{N}(0,25)$.}
\label{table:freq-bias-3}
\begin{tabular}{|c|c|c|c|c|}
\hline
 & $m=0$           & $m=1$           & $m=2$      &$m=3$             \\ \hline
Greedy            & 0.05 & 0.26 & 0.44  &0.25     \\ \hline
lil' UCB            &0.07  &0.32  &0.43 & 0.18      \\ \hline
$\epsilon$-Greedy &0.09  &0.34  &0.40  &0.17       \\ \hline
TS  &0.07  &0.31  &0.43  &0.20      \\ \hline
\end{tabular}
\end{table}
We have included additional experiment results to supplement Table~\ref{table:freq-bias} that confirm the findings in Section~\ref{sec:bias-analysis}. In Table~\ref{table:freq-bias-2} and Table~\ref{table:freq-bias-3}, we run experiments in settings there are two and three distributions respectively (details see captions of the Tables). Each column is the fraction of trials in which $m$ distributions have negative bias, where $m=0,1,2$ in the case of two distributions, and $m=0,1,2,3$ in the case of three distributions. 
\section{Examples of computing the conditional likelihood}
\label{app:example-condi}
Here are some examples of computing the explicit forms of the conditional likelihood. 
We see from Eqn.~\ref{eq:ll} that it suffices to compute the selective likelihood ratios through Eqn.~\ref{eq:selective:lr}
for the different algorithms. The explicit form of the selection probability for Thompson Sampling can be found in Appendix~\ref{app:thompson}.

\begin{enumerate}
\item {\bf lil' UCB + Gumbel} or {\bf Greedy + Gumbel}: per Lemma \ref{lem:gumbel},
$$
\Parg{f(\mathbf{U}_t) = k \mid \mathbf{U}_t} = \frac{\exp\left[U_t^{(k)} / \tau_t\right]}{\sum_{i=1}^K \exp\left[U_t^{(i)}/\tau_t\right]}.
$$ 
\item {\bf $\epsilon$-Greedy}:
$$
\Parg{f(\mathbf{U}_t) = k \mid \mathbf{U}_t} = \frac{\epsilon}{K} + (1-\epsilon) \indic{\argmax_i \mean{X_t^{(i)}} = k}.
$$
{\bf $\epsilon$-Greedy + Gumbel}: the selection function will be
$$
f(\mathbf{U}_t) = \begin{cases}
\argmax_{k} \mean{X_t^{(k)}} + \epsilon_t^{(k)}, & \text{ w.p. } 1-\epsilon \\
k, k\in[K] & \text{ w.p. } \frac{\epsilon}{K}
\end{cases}, \quad
\epsilon_t^{(k)} \overset{\textrm{iid}}{\sim} G_{\tau_t}.
$$
and the selection probabilities are
$$
\Parg{f(\mathbf{U}_t) = k \mid \mathbf{U}_t} = \frac{\epsilon}{K} + (1-\epsilon) \cdot \frac{\exp[\mean{X_t^{(k)}} / \tau_t]}{\sum_{i=1}^K \exp[\mean{X_t^{(i)}}/\tau_t]}.
$$
We see that with Gumbel randomization, the only difference is that we replace argmax with the
softmax function.

\end{enumerate}
\section{Optimization the cMLE with contrastive divergence} \label{app:cmle-cd}
Theorem \ref{thm:ll} gives an explicit formula for likelihood function up to a normalizing
constant (partition function). Since it is infeasible to get an explicit formula for this partition function,
we use Contrastive Divergence (CD) proposed in \citet{contrastive_divergence} 
for solving the Maximum Likelihood Estimation problem. 

To maximize the log-likelihood, 
$$
\max_{\theta} ~\log p(\Lambda_T \mid s_t,~t=1,\dots,T; \theta),
$$
we compute its approximate gradient descent using CD. Suppose 
$$
p(\Lambda_T \mid s_t,~t=1,\dots,T; \theta) = 
\frac{\ell(\Lambda_T \mid s_t,~t=1,\dots,T; \theta)}{Z(\theta)},
$$
then the approximate gradient step for $\theta$ would be
$$
\theta_{i+1} = \theta_i + \eta\left(\frac{\partial \ell}{\partial \theta}\bigg|_{\Lambda_T}
- \frac{\partial \ell}{\partial \theta}\bigg|_{\Lambda_T'}\right),
$$ 
where $\Lambda_T'$ is a single step of MCMC from the density
$p(\Lambda_T \mid s_t,~t=1,\dots,T; \theta_i)$, $\eta$ is the step size.
Contrastive Divergence can be seen as a form of stochastic gradient descent where
the gradient $\frac{\partial \log Z(\theta)}{\partial \theta} = 
\Esubarg{\Lambda_T}{\frac{\partial \ell}{\partial \theta}}$ is approximated by a single
sample from the MCMC chain. In practice, to stabilize the gradient, we may take multiple samples
from the MCMC chain and average the gradient to reduce variance. 

See Algorithm~\ref{alg:CD} for finding the cMLE using Contrastive Divergence.
\begin{algorithm}
  \caption{Algorithm for computing cMLE for adaptive data collection}
  \label{alg:CD}

  \begin{algorithmic}
    \State Initialize $\theta_0 = \left(\mean{X_T^{(1)}}, \dots, \mean{X_T^{(K)}}\right)$ to be the empirical means. 
    \Repeat
    \State Obtain MCMC samples $(\Lambda_T^{'(1)}, \dots, \Lambda_T^{'(R)})$ from the density in Eqn.~\ref{eq:ll} at $\theta_i$, where $R$ is the number of MCMC samples we take.
    \State Update $\theta$ through the gradient step,
    $$
    \theta_{i+1} = \theta_i + \eta\left(\frac{\partial \ell}{\partial \theta}\bigg|_{\Lambda_T}
    - \frac{1}{R}\sum_{r=1}^R\frac{\partial \ell}{\partial \theta}\bigg|_{\Lambda_T'^{(r)}}\right),
    $$ 
    \State $i \mapsto i+1$
    \Until $\theta_i$ converges.
  \end{algorithmic}
\end{algorithm}

\section{Gumbel-Max trick} \label{app:gumbel}
\begin{lemma}[Gumbel-Max trick]
\label{lem:gumbel}
For any fixed vectors $U = (U^{(1)}, \dots, U^{(K)}) \in \mathbb{R}^{K}$, we have
$$
\Psubarg{\epsilon}{\argmax_i U^{(i)} + \epsilon^{(i)} = k} = \frac{\exp(U^{(k)} / \tau)}{\sum_{i=1}^K \exp(U^{(k)} / \tau)},
$$
where $\epsilon^{(k)} \overset{\textrm{iid}}{\sim} G_{\tau}$, where $G_{\tau}$ is Gumbel distribution with scale $\tau$.
\end{lemma}

\begin{proof}
Let $t(x) = \exp(-x/\tau)$, then we have
\begin{align*}
&\Psubarg{\epsilon}{U^{(k)} + \epsilon^{(k)} > U^{(i)} + \epsilon^{(i)}, ~ i \neq k}\\
=&\Psubarg{\epsilon^{(k)}}{\prod_{1\leq i\leq K, i\neq k}e^{-t(U^{(k)}+\epsilon^{(k)}-U^{(i)})}}\\
	=&\int\limits_{\epsilon^{(k)}\in\mathbb{R}} \exp\left(-\sum_{1\leq k\leq K, i\neq k}t(U^{(k)}+\epsilon^{(k)}-U^{(k)})\right)\frac{1}{\tau}t(\epsilon^{(k)})e^{-t(\epsilon^{(k)})}d\epsilon^{(k)} \\
	=&\int\limits_{\epsilon^{(k)}\in\mathbb{R}} \exp\left(-\sum_{i=1}^K t(\epsilon^{(k)}+U^{(k)}-U^{(i)})\right)\frac{1}{\tau}t(\epsilon^{(k)})d\epsilon^{(k)} \\
	=&\int\limits_{\epsilon^{(k)}\in\mathbb{R}} \exp\left(-t(\epsilon^{(k)})\sum_{i=1}^K t(U^{(k)}-U^{(i)})\right)\frac{1}{\tau}t(\epsilon^{(k)})d\epsilon^{(k)} \\
	=&-\int\limits_{-\infty}^0 \exp\left(-s\sum_{i=1}^K t(U^{(k)}-U^{(i)})\right)ds \\
	=&\frac{1}{\sum_{i=1}^K t(U^{(k)}-U^{(i)})} =\frac{e^{U^{(k)}/\tau}}{\sum_{i=1}^K e^{U^{(i)}/\tau}}.
\end{align*}
\end{proof}

\section{Propensity Matching}
\label{app:prop}
Propensity Matching \citep{austin2011introduction} is an unbiased estimator that is commonly used in selection functions that make choices based on the probability of selecting a distribution, such as in EXP3 suggested by \citet{auer2002nonstochastic}. The estimator achieves consistent estimates by 
\begin{eqnarray}
\hat{\mu}^{(k)} = \frac{\sum_{t=1}^T\indic{f(\Lambda_t) = k} \cdot \frac{X_{N_t^{(k)}}^{(k)}}{\Pr[f(\Lambda_t) = k]}}{T}.
\end{eqnarray}
for $k\in [K]$, where $T$ is the horizon.
This estimator also suffers from high variance, as observed in Table~\ref{table:mse-full}. Additionally, this estimator is only relevant to be applied if the selection function $f$ outputs a probability distribution over which one of the $K$ distributions to select at each time step, so it is not readily applicable to Greedy and UCB type algorithms.

\section{Extensions to Thompson Sampling}
\label{app:thompson}
Thompson Sampling is another common bandit algorithm \citep{thompson1933likelihood,agrawal2012analysis}. We extend Proposition~\ref{prop:algs_satifies} to Thompson sampling, and then show how to apply cMLE, and finally show empirical results.
In the Gaussian setting,
with Gaussian prior $\mu^{(k)} \sim \mathcal{N}(\mu_0^{(k)}, \sigma_0^2)$, and $X_t^{(k)} \sim \mathcal{N}(\mu^{(k)}, \sigma^2)$, where $\mu^{(k)}$ is the true mean for distribution $k$, the decision statistics are the posterior means and variances,
$$
\begin{gathered}
U_t^{(k)} = (\mu_t^{(k)}, \sigma^{(k) 2}_t)\\
\mu_t^{(k)} = \displaystyle\frac{\left(\frac{\mu_0}{\sigma_0^2} + \frac{N_t^{(k)}\mean{X_t}^{(k)}}{\sigma^2}\right)}{\frac{1}{\sigma_0^2} + \frac{N_t^{(k)}}{\sigma^2}},
\quad \sigma_t^{(k)2} = \left(\frac{1}{\sigma_0^2} + \frac{N_t^{(k)}}{\sigma^2}\right)^{-1}.
\end{gathered}
$$ 
The selection function is
$$
f(\mathbf{U}_t) = \argmax_{k} \hat{\mu}_t^{(k)}, \quad \hat{\mu}_t^{(k)} \sim N(\mu_t^{(k)}, \sigma_t^{(k)2}).
$$
\subsection{Extension of Proposition~\ref{prop:algs_satifies} to Thompson Sampling}
\begin{lemma}
Let $\bm{\theta}^{(k)}=\{\theta_i^{(k)}\}$ be a set of $M$ parameters that are updated after each drawing of distribution $k$. Let $F_{\bm{\theta}^{(k)}}$ be the CDF of $\theta_i^{(k)}$. Define the generalized inverse CDF $F^{-1}_{\theta_i^{(k)}}(q) = \inf \{\theta \in \R:F(\theta)\geq q\}$. Assume for any $q_1, \cdots, q_M \in [0,1]$,
\begin{eqnarray}
\E\left[X^{(k)}|F^{-1}_{\theta_1^{(k)} \mid \mean{X}_t^{(k)}}(q_1), \cdots, F^{-1}_{\theta_M^{(k)} \mid \mean{X}_t^{(k)}}(q_M)\right] \geq \E\left[X^{(k)}|F^{-1}_{\theta_1^{(k)} \mid \mean{X}_t^{(k)'}}(q_1), 
\cdots, F^{-1}_{\theta_M^{(k)} \mid \mean{X}_t^{(k)'}}(q_M)\right] \label{eq:thom-cdf}
\end{eqnarray}
if 
$\mean{X}_t^{(k)} > \mean{X}_t^{(k)'}$.
Then Thompson sampling is also equivalent to selection function $f(\Lambda_t, \omega =\{q_i\}_{i=1}^M)$ that satisfies \emph{Exploit} and \emph{IIO}.
\begin{proof}
Recall that in Thompson Sampling, we choose the distribution that has the highest expected mean conditioned on a sample drawn from the posterior distribution of $\bm{\theta}^{(k)}$. Since we condition on a fixed realization of random seeds $q_1, \cdots, q_M$ drawn to sample from the inverse CDF of the posterior of $\bm{\theta^{(k)}}$, Equation \eqref{eq:thom-cdf} implies that the expected mean is higher for distributions that have higher empirical mean so far. \emph{Exploit} is satisfied by definition. For \emph{IIO}, since the posterior of $\bm{\theta}^{(k)}$ is a deterministic function of the history $\Lambda^{(k)}$, it is also trivially satisfied.
\end{proof}
\end{lemma}
\subsection{cMLE for Thompson Sampling}

For {\bf Thompson + Gumbel}, additional Gumbel noises are added to the sampled expected reward $\hat{\mu}_t^{(k)}$'s.
In other words, the selection function will be
$$
f(\mathbf{U}_t) = \argmax_{k} \hat{\mu}_t^{(k)} + \epsilon_t^{(k)}, \quad \hat{\mu}_t^{(k)} \sim N(\mu_t^{(k)}, \sigma_t^{(k)2}), \quad
\epsilon_t^{(k)} \overset{\textrm{iid}}{\sim} G_{\tau_t},
$$
where $G_{\tau_t}$ is a centered Gumbel distribution with mean 0 and scale $\tau_t$.

The selection probability
$$
\Parg{f(\mathbf{U}_t) = k \mid \mathbf{U}_t} =  \prod_{t=K}^{T-1}\prod_{k=1}^K 
\phi\left(\frac{\hat{\mu}_t^{(k)} - \mu_t^{(k)}}{\sigma_t^{(k)}}\right)
\prod_{t=K}^{T-1} \frac{\exp[\hat{\mu}_t^{(k)} / \tau_t]}{\sum_{i=1}^K \exp[\hat{\mu}_t^{(i)}/\tau_t]},
$$
where the softmax terms come from the additional Gumbel randomizations.
\subsection{Experimental results}
We compare the bias and MSE of the original Thompson Sampling (TS) algorithm, and the debiased results after running cMLE. The debiasing runs 3000 gradient descent steps, 30 steps of MCMC with the first half as burn-in. The scale of the Gumbel distribution is 1.0. See Table~\ref{table:thom} for experiment results.
\\
\makeatletter% Set distance from top of page to first float
\setlength{\@fptop}{5pt}
\makeatother
\begin{table}[!htb]
\centering
\caption{In the left table, we compare the bias of the original Thompson Sampling (TS) algorithm and the bias after running cMLE, for $K=2$ and $K=5$ distributions, with different stopping values T. With $K=2$, each distribution is drawn from $\mathcal{N}(\mu_i,1)$. where $\mu_1=1.0, \mu_2=0.75 $. With $K=5$, each distribution is drawn from $\mathcal{N}(\mu_i,1)$. where $\mu_1=1.0, \mu_2=0.75, \mu_3=0.5, \mu_4=0.38, \mu_5=0.25.$ All distributions have prior $\mathcal{N}(0,25)$. The left column is the bias of the original algorithm, and the right column is the percentage of bias that is left after running cMLE. In the right table, we compare the MSE of the original algorithm, data splitting (held-out), and cMLE. The leftmost columns show the MSE in the original algorithm, and the right two columns show the percentage in comparison to the MSE of the original algorithm. We see that data splitting suffers from high variance, and cMLE improves MSE.}
\label{table:thom}
\begin{tabular}{|c|c|c|}
\hline
\multirow{2}{*}{} & \multicolumn{2}{c|}{TS}\\ \cline{2-3}
& orig. & cMLE\\ \hline
T=24,K=2 & -0.19 & 18.7\%\\ \hline
T=32,K=2 & -0.17 & 20.5\%\\ \hline
T=60,K=5 & -0.23 & 37.3\%\\ \hline
T=80,K=5 & -0.11 & 28.8\%\\ \hline
\end{tabular}
\quad
\begin{tabular}{|c|c|c|c|}
\hline
\multirow{2}{*}{} & \multicolumn{3}{c|}{TS}\\ \cline{2-4}
& orig. & held-out  & cMLE\\ \hline
T=24,K=2 & 0.32 & 130.0\% & \textbf{90.0\%}\\ \hline
T=32,K=2 & 0.28 & 110.0\% & \textbf{77.0\%}\\ \hline
T=60,K=5 & 0.34 & 123.0\% & \textbf{85.0\%}\\ \hline
T=80,K=5 & 0.16 & 125.0\% & \textbf{62.0\%}\\ \hline
\end{tabular}
\end{table}

\end{document}